\documentclass{article}

\PassOptionsToPackage{numbers,compress}{natbib}
 \usepackage[preprint]{neurips_2026}

\usepackage[utf8]{inputenc} 
\usepackage[T1]{fontenc}    
\usepackage[pagebackref,breaklinks,colorlinks]{hyperref}
\usepackage{url}            
\usepackage{booktabs}       
\usepackage{amsfonts}       
\usepackage{nicefrac}       
\usepackage{microtype}      
\usepackage[table]{xcolor}         

\usepackage{mdframed}
\usepackage{lipsum}
\usepackage{amsmath}
\usepackage{geometry}
\usepackage{ulem}
\usepackage{algorithm}
\usepackage{amsthm}
\usepackage{algpseudocode}
\usepackage{xcolor,cancel}
\usepackage{xfrac}
\usepackage{dsfont}
\usepackage{caption}

\usepackage{amssymb}
\usepackage{multirow}
\usepackage{wrapfig}

\definecolor{NavyBlue}{rgb}{0.0, 0.0, 0.5}
\definecolor{richelectricblue}{rgb}{0.03, 0.57, 0.82}
\definecolor{baseline}{gray}{0.90}
\newmdtheoremenv{theo}{Theorem}
\definecolor{ranjay}{RGB}{0, 80, 103}

\definecolor{george}{RGB}{0, 212, 255}

\newcommand{\titlespace}{$\quad$}
\newcommand{\titleheight}{\vspace{3pt}}

\title{Posterior Augmented Flow Matching}

\author{%
  \centerline{George Stoica$^{1,2,\dagger}$\titlespace{}\titlespace{} Sayak Paul$^{3,\diamondsuit}$\titlespace{}\titlespace{} Matthew Wallingford$^{4,\diamondsuit}$} \\
  \centerline{Vivek Ramanujan$^{2}$\titlespace{}\titlespace{} Abhay Nori$^{2}$\titlespace{}\titlespace{} Winson Han$^{3}$} \\
  \centerline{Ali Farhadi$^{2}$\titlespace{}\titlespace{} Ranjay Krishna$^{2}$\titlespace{}\titlespace{} Judy Hoffman$^{1 5}$} \titleheight{} \\
  \centerline{$^{1}$Georgia Tech $^{2}$University of Washington $^{3}$Hugging Face $^{4}$Ai2 $^5$UC Irvine} \titleheight{} \\
  \centerline{$^\dagger$Correspondence to: \texttt{gstoica3@gatech.edu} \titlespace{} $^\diamondsuit$Equal Contribution} \\
}
\begin{document}

\maketitle

\begin{abstract}
    Flow matching (FM) trains a time-dependent vector field that transports samples from a simple prior to a complex data distribution. However, for high-dimensional images, each training sample supervises only a single trajectory and intermediate point, yielding an extremely sparse and high-variance training signal. 
This under-constrained supervision can cause \textit{flow collapse}, where the learned dynamics memorize specific source–target pairings, mapping diverse inputs to overly similar outputs, failing to generalize. 
We introduce \textit{Posterior-Augmented Flow Matching} (PAFM), a theoretically grounded generalization of FM that replaces single-target supervision with an expectation over an approximate posterior of valid target completions for a given intermediate state and condition.
PAFM factorizes this intractable posterior into (i) the likelihood of the intermediate under a hypothesized endpoint and (ii) the prior probability of that endpoint under the condition, and uses an importance sampling scheme to construct a mixture over multiple candidate targets. 
We prove that PAFM yields an unbiased estimator of the original FM objective while substantially reducing gradient variance during training by aggregating information from many plausible continuation trajectories per intermediate. 
Finally, we show that PAFM improves over FM by up to 3.4 FID50K across different model scales (SiT-B/2 and SiT-XL/2), different architectures (SiT and MMDiT), and in both class and text conditioned benchmarks (ImageNet and CC12M), with a negligible increase in the compute overhead.
Code: \url{https://github.com/gstoica27/PAFM.git}.
\end{abstract}
\begin{figure*}[htbp]
  \centering
  \includegraphics[width=\textwidth]{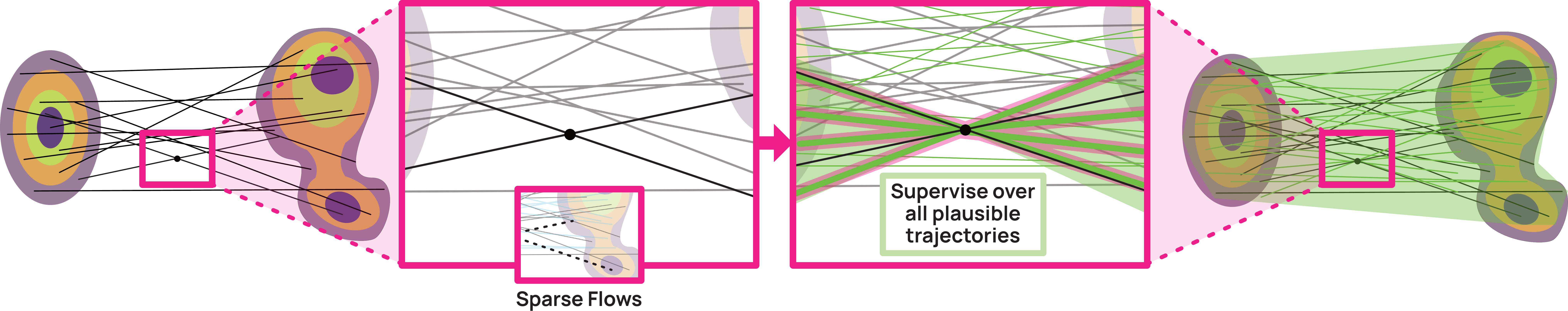}
  \caption{
    \textbf{Left}: Standard FM provides a sparse, one-to-one supervision signal, pairing each intermediate point with a single target and yielding a single plausible flow.
    \textbf{Right}: Our PAFM aggregates supervision over the full posterior of compatible targets, producing a denser, more coherent set of plausible flows from each intermediate.}
  \label{figure}
\end{figure*}
\section{Introduction}

Flow matching (FM)~\citep{lipman2022flow, ma2024sitexploring, stoica2025contrastive, yu2024representation} trains a vector field that transports probability mass from a simple source distribution (e.g., Gaussian noise) to a complex target distribution (e.g., natural images), often conditioned on auxiliary inputs that specify \textit{how} to flow to generate samples matching the given condition (e.g., ``a photo of a dog'')~\cite{esser2024scaling, ma2024sit}.
Concretely, training a FM model proceeds as follows~\cite{lipman2022flow}: sample a data example and its condition from the target distribution, sample a point from the source distribution, and define a continuous interpolation (trajectory) between them.
A random intermediate point along this trajectory is then sampled, and the model is asked to predict the velocity that moves this intermediate toward the target, conditioned on both the intermediate’s location and the condition.
Over many such samples, the model learns to approximate a global flow field that, in expectation, carries any source sample to a valid target consistent with the condition.

Under idealized assumptions—unlimited data, infinite model capacity, and perfect optimization—minimizing the FM loss recovers the true underlying vector field, and thus the data distribution, exactly~\cite{lipman2022flow, fukumizu2024flow}.
In practice, however, models are trained on complex, high-dimensional data for a finite number of steps and with bounded capacity.
They observe only a vanishingly small subset of all possible trajectories, resulting in \textit{sparse supervision} of the flow field~\cite{huang2025improving}.

This sparsity is inherent in the FM loss: each training sample provides feedback on only a \textit{single} trajectory (the one connecting the chosen source and target) at a given intermediate point~\cite{lipman2022flow}.
In high-dimensional spaces, a single intermediate can lie on many plausible trajectories leading to different valid targets under the same condition—for instance, there are countless distinct images consistent with the prompt ``a photo of a dog''~\cite{aggarwal2001surprising}.
Yet the model only ever sees one such outcome.
Because the probability of revisiting the exact same intermediate state is effectively zero, the gradient at that state is estimated from a single noisy supervision signal.
This under-constrained, high-variance training signal can lead to \textit{flow collapse}~\cite{reu2025gradient}, where the learned vector field fails to generalize and instead maps diverse inputs to overly similar or muddled outputs, indicating memorization of specific training pairings rather than learning the true continuous mapping.

We propose \textbf{Posterior-Augmented Flow Matching (PAFM)}, a theoretically grounded generalization of the standard FM objective.
The key insight is that for any intermediate state and condition, there is not a single ``correct'' target data point, but rather an entire \textit{posterior distribution} over valid targets in the data distribution. In other words, there are many plausible ways to complete the flow from the intermediate to a final sample that meets the condition. 
\textit{All such trajectories represent valid training signals for the model}. 
However, these possible flows are not equally likely: some target completions have higher posterior probability than others. 
Therefore, PAFM trains the model to predict the expected velocity over all these possible continuation trajectories, \textit{weighted} by their posterior likelihood. In effect, the model learns to follow the average direction toward the full set of outcomes consistent with the condition, instead of being pushed toward only \textit{one} arbitrarily sampled endpoint. 

While we cannot directly sample uniformly from the intractable true posterior over continuations, we show that it can be factorized into two simpler distributions that we can approximate: (1) the conditional probability of a given intermediate state if a particular target were the endpoint, and (2) the probability of the conditioning variable given that target. 
We show that under certain formulations, a sampling algorithm which draws a set of candidate targets for each intermediate and weighs them according to these probabilities forms an approximate posterior mixture.
Moreover, by aggregating gradients from multiple possible flows (each weighted by its posterior likelihood) for each intermediate point, 
PAFM lower-bounds the variance of the flow matching training gradient

In addition to our theoretical contributions, we demonstrate the applicability of PAFM in real-world settings and model-architectures.
We conduct experiments on the popular class-conditioned ImageNet-1K dataset~\cite{imagenet} across different model scales (SiT-B/2 and SiT-XL/2~\cite{ma2024sit}), and on the large-scale text-to-image CC12M~\cite{sharma2018conceptual} dataset with the MMDiT~\cite{dit} model architecture.
We investigate several practical strategies for constructing the candidate target set, including FAISS~\cite{douze2024faiss} k-nearest-neighbor retrieval in the latent space, augmentation of the source image via random crops at varying scales, and resampling from the VAE~\cite{ldm} moment distribution. 
PAFM improves flow matching across all settings and selection strategies while incurring only marginal computational overhead, highlighting its seamless integration into modern generation workflows. 



\section{Related Work}\label{sec:related_work}

\noindent\textbf{Continuous-time generative modeling and flow matching.}
Continuous-time generative models view sampling as integrating an ODE or SDE between a simple source distribution and a complex target distribution, covering stochastic interpolants, diffusion models, and flow-based methods~\cite{albergo2023stochastic,ho2020ddpm,song2021score}. 
Flow matching (FM) and conditional flow matching (CFM) train a vector field by matching it to the ground-truth probability current along prescribed interpolants between source and data samples~\cite{lipman2022flow,tong2024cfm}. 
Subsequent work has scaled FM and CFM to large image and video models by carefully designing architectures and interpolants, including interpolant transformers such as SiT~\cite{ma2024sit} and large-scale rectified-flow-style models~\cite{liu2022rectified,esser2024rectified}. 
These methods all share the same supervision pattern: each intermediate point along an interpolant is paired with a \textit{single} target sample, yielding a one-to-one training signal.

\noindent\textbf{Rectified flows and trajectory shaping.}
Rectified flows encourage straight trajectories between source and target, enabling near one-step generation and linking flow-based methods to optimal transport formulations~\cite{liu2022rectified,albergo2023stochastic}. At scale, rectified-flow transformers match or surpass diffusion models on high-resolution text-to-image benchmarks via improved time sampling and architecture design~\cite{sd3}, while complementary work studies how controlling trajectory curvature can reduce solver error and accelerate sampling~\cite{lee2023curvature}. Our formulation is compatible with these trajectory-shaping ideas: PAFM leaves the interpolant and geometric regularization unchanged and instead modifies how supervision is aggregated over targets.

\noindent\textbf{Theoretical analyses and improved flow matching objectives.}
The FM objective enjoys strong asymptotic guarantees: under mild assumptions on vector-field parameterization and sample complexity, FM attains near minimax-optimal convergence rates in Wasserstein distance, comparable to diffusion models~\cite{fukumizu2024flow}. Follow-up work tightens the link between learned and ideal probability paths, motivating divergence-matching augmentations to CFM~\cite{huang2025divergence} and relating rectified flows to optimal transport via straightening and gradient constraints~\cite{hertrich2025rectified}. These analyses refine global objectives and path properties but still assume the standard one-target-per-intermediate coupling.

\noindent\textbf{Failure modes and sparse supervision in flow-based models.}
In realistic regimes, models are finite-capacity, trained for finite steps, and only observe a sparse subset of source–target trajectories~\cite{ddim}. Recent empirical and theoretical work has exposed failure modes of flow-based models under these conditions, including degeneracies induced by straight-path objectives and deterministic couplings~\cite{reu2025gradient}. Gradient-variance analyses further show that rectified-flow-style objectives can memorize arbitrary pairings in low-stochasticity regimes, even when interpolant lines intersect, leading to ill-defined vector fields at inference~\cite{reu2025gradient}. Our work is complementary: starting from the same observation that supervision is sparse and under-constrained, we replace one-to-one supervision with a posterior expectation over all compatible targets for each intermediate latent.

\section{Background: Flow Matching}
\label{sec:background}

We provide a formal overview of conditional flow matching and the flow matching (FM) objective as it pertains to our work.

\noindent\textbf{Assumptions.} Let $\mathcal{N}(0, I_d)=p_{source}$ be a source distribution characterized by the standard Gaussian.
Similarly, let $p_{data}$ be the target data distribution with unknown support.
Samples drawn from $p_{source}$ are denoted by $\epsilon^i$, while those from $p_{data}$ are denoted by $(z^i, y^i)$ pairs, where $z^i$ is a data point and $y^i$ is its associated condition. For our purposes, $z^i$ represents an image and $y^i$ a text conditioning.
We may equivalently sample $z^i\sim p_{data}(\cdot|y^i)$ or $y^i\sim p_{data}(\cdot | z^i)$.

Flow matching describes the flow between $p_{source}$ and $p_{data}$ over \textit{normalized and reversed} time.
Let $\text{Unif}[0, 1)$ describe the distribution over this time, and denote its samples by $t$.
Training FM models involves sampling from the source and target distribution (e.g., $\epsilon^i$ and $(z^i, y^i)$), defining a trajectory between their respective samples and then further sampling an intermediate point which lies along their flow.

Let $\alpha(t), \beta(t)$ be two time-dependent monotonic functions that define the path between $\epsilon^i$ and $z^i$, such that $\alpha(0)=\beta(1)=0$ and $\alpha(1)=\beta(0)=1$.
For notational simplicity, we refer to point-evaluations at $t$ by $\alpha_t$ and $\beta_t$ respectively.
Using $\alpha_t$ and $\beta_t$, we sample the intermediate point as $z_t^i=\alpha_t \epsilon^i + \beta_t z^i$.
These samples are assumed to lie on a conditional probability path, given by $z_t^i\sim p_t(\cdot|z^i) = \mathcal{N}(\beta_t z^i, \alpha_t^2 I_d)$.
Observe that $p_0(\cdot|z^i)=p_{source}$ and $p_1(\cdot|z^i)=\delta_{z^i}$ where $\delta_{z^i}$ is the ``Dirac delta'' distribution.
We characterize the flow between $\epsilon^i$ and $z^i$ by the time-derivative of $\alpha_t$ and $\beta_t$ respectively: $v(z_t^i|z^i)=\dot{\alpha}_t \epsilon^i + \dot{\beta}_t z^i$. 
We describe the vector-space over which $\epsilon^i, z^i, z_t^i$ are supported, as the ``latent-space.''

\noindent\textbf{The flow matching objective.}
While $\alpha, \beta$ can take arbitrary forms, $\alpha(t)=t$ and $\beta(t)=1-t$ (with $\dot{\alpha}(t)=1, \dot{\beta}(t)=-1$) is the most popular and widely used formulation~\cite{ma2024sit, yu2024representation}.
Thus, we define $z_t^i=t\epsilon^i + (1-t)z^i$ and $v(z_t^i|z^i)=\epsilon^i - z^i$ respectively.
FM involves training a model to produce flows from arbitrary $z_t^i$ to real targets (e.g., $z^i$) based on a condition (e.g., $y^i$).
Let $f_\theta(z_t^i|t,y^i)$ describe this model and let $\theta$ be its parameterization.
The flow matching objective is finally given by,
\begin{equation}
    \mathcal{L}^{(FM)}(\theta) = \mathbb{E}_{
    t\sim Unif[0,1), \\ (z^i, y^i)\sim p_{data}(\cdot), z_t^i\sim p_t(\cdot|z^i)}|| f_\theta(z_t^i| t, y^i) - v(z_t^i|z^i) ||^2\label{eq:fm_theory}
\end{equation}
Thus, FM models learn the \textit{expected} flows that pass through each intermediate sample towards the entire target distribution, conditioned on each $y^i$. 
In practice, the formal objective in~Eq.~\ref{eq:fm_theory} is transformed to,
\begin{equation}
    \min \overline{\mathcal{L}^{(FM)}}(\theta) = \min \frac{1}{N}\sum_{(z^i,y^i)\sim D} || f_\theta(z_t^i| t, y^i) - v(z_t^i|z^i) ||^2_2\label{eq:fm_training}
\end{equation}
when training flow models, where $D$ is the target dataset. 

\noindent\textbf{Sparse supervision.}
While~Eq.~\ref{eq:fm_training} is an unbiased estimator for~Eq.~\ref{eq:fm_theory} and its global minima is the global optimum under flow matching, we immediately observe its inherent risk of sparse supervision.
For any given intermediate latent point $z_t^i$ and given condition $y^i$, a model is supervised with only a single target trajectory $v(z_t^i|z^i)$.
However, in the high-dimensional and complex settings of contemporary generation tasks, $z_t^i$ could validly lie on a multitude of trajectories flowing to different targets \{$z^j\}$ that all satisfy the same condition $y^i$.
By only providing one such target,~Eq.~\ref{eq:fm_training} gives a high-variance and under-constrained signal over each latent.
As a result, models often rely on sparse coverage of the true conditional vector fields within the latent space, leading to collapsed flows when transporting from source to target data~\cite{agrawalla2025floq}.

\section{Posterior-Augmented Flow Matching (PAFM)}\label{sec:method}
The focus of this section is two-fold.
First, we show how the problem of \textit{sparse supervision} can be addressed by reformulating the flow matching objective to enable simultaneous \textit{multiple} trajectory supervision for training flow models.
We coin this new objective Posterior-Augmented Flow Matching (PAFM), and prove that it improves gradient stability over flow matching under the \textit{same} theoretical assumptions.
Second, we show how PAFM can be seamlessly adapted for training rectified flow models. 
\subsection{Theoretical Results}\label{subsec:pafm_theory}
\noindent\textbf{Reformulating the flow matching objective.}
We begin by observing that it is theoretically feasible for multiple samples drawn from $p_{source}$ and $p_{target}$ to contain paths which flow through the same $z_t^i$.
Moreover, the likelihood of sampling these trajectories is given by a posterior-distribution over the \textit{same} distributions supporting the flow matching objective.
Coupling these observations together, we introduce Posterior-Augmented Flow Matching (PAFM) as,
\begin{equation}
    \mathcal{L}^{(PAFM)}(\theta) = \mathbb{E}_{\substack{
    t\sim Unif[0,1),
    y^i\sim p_{\text{data}}(y), \\
    z_t^i\sim p_t(\cdot),
    z^j\sim p_t(\cdot|z_t^i, y^i)}}
    ||f_\theta(z_t^i| t, y^i) - v(z_t^i|z^j) ||^2\label{eq:pafm_theory}
\end{equation}
where $p_{data}(y)$ defines the probability distribution over the conditions, $p_t(\cdot)$ is a (unknown) probability distribution over the latent-space conditioned on $t$, and $p_t(\cdot |z_t^i,y^i)$ is the \textit{posterior} distribution of the target data given a \textit{fixed} latent-point $z_t^i$ and condition $y^i$.

Importantly,~Equation~\ref{eq:pafm_theory} is mathematically equivalent to the flow matching objective introduced by~Equation~\ref{eq:fm_theory}.
We demonstrate this in Theorem~\ref{theorem:equivalence}.

\begin{theo}\label{theorem:equivalence}
The Posterior-Augmented Flow Matching objective $\mathcal{L}^{(PAFM)}$ is an \textit{unbiased} estimator of the standard Flow Matching objective $\mathcal{L}^{(FM)}$.
In expectation, the two objectives are identical.
\begin{equation}
    \mathcal{L}^{(PAFM)}(\theta) = \mathcal{L}^{(FM)}(\theta)\label{eq:thorem1}
\end{equation}
\end{theo}
\begin{proof}
It suffices to show that the joint probability distributions over which the expectations are taken over are identical. 
FM (~Equation~\ref{eq:fm_theory}) takes its expectation over the joint probability distribution given by $p(t)\cdot p_{data}(z,y)\cdot p_t(z_t|z)$.
The PAFM (~Equation~\ref{eq:pafm_theory}) takes its expectation over the joint distribution given by $p(t)\cdot p_t(z_t)\cdot p_t(z,y|z_t)$.
Note that by the flow matching assumptions introduced in~Section~\ref{sec:background}, $z \perp y | z_t$.
Starting from the distributions for FM,
\begin{align}
    p(z,y)p_t(z_t|z) &= \frac{p(z,y) p(z_t,z, t)}{p(z, t)} = \frac{p(y|z)p(z_t, t|z)p(z)^2}{p(z)p(t)} \\ 
     &= \frac{p(y|z)p(z_t, t|z)p(z)}{p(t)} = \frac{p(y, z_t, t, z)}{p(t)} \\
      &= \frac{p(z|z_t, y, t)p(z_t, y, t)}{p(t)} = p_t(z|z_t, y)p_t(z_t) p(y)
\end{align}
As the two distributions are identical, so is their expectation over the same loss function $||f_\theta(z_t| t, y) - v(z_t|z) ||^2$. $\square$
\end{proof}

\noindent\textbf{PAFM theoretically reduces gradient variance.}
Notably, and in contrast to flow matching,~Equation~\ref{eq:pafm_theory} enables us to optimize $f_\theta$ by sampling multiple target points $\{z^j\}$ for every $z_t^i$, yielding a \textit{denser} supervision signal during training.
We now demonstrate that this fact alone provides gradient estimates for the underlying optimal $f_\theta$ that \textit{lower-bound} the variance of the gradient estimates from the FM objective.

\begin{theo}\label{theorem:variance}
Let $\phi(z^j|z_t^i, y^i)=||f_\theta(z_t^i|y^i, t) - v(z_t^i|z^j)||^2$ be the per-sample flow matching loss.
Define its gradient with respect to $f_\theta$ as $g(z^j|z_t^i, y^i) = \nabla_f\phi(z^j|z_t^i, y^i)=2\left(f_\theta(z_t^i|y^i,t) - v(z_t^i|z^j)\right)$.
For a fixed $z_t^i, y^i$, let $Var_{p_t(\cdot | z_t^i, y^i)}$ $\left[g(z^j|z_t^i, y^i) \right] = \Sigma_g(z_t^i)$.
The PAFM objective, when optimized using $K > 1$ samples per $z_t^i$ via Self-Normalized Importance Sampling (SNIS), yields a gradient estimator with variance $\sfrac{\Sigma_g}{ESS(z_t^i)}$, where $ESS(z_t^i)\geq 1$ is the Kish Effective Sample Size~\cite{wiegand1968kish}.
\end{theo}
\begin{proof}
First, notice that,
\begin{align}
    p_t(z^j|z_t^i, y^i) &= \frac{p_t(z^j, z_t^i, y^i)}{p_t(z_t^i, y^i)} = \frac{p_t(z_t^i, y^i|z^j)p(z^j)}{p_t(z_t^i, y^i)} = \frac{p_t(z_t^i|z^j)p_t(y^i|z^j)p(z^j)}{p_t(z_t^i, y^i)} \\
    &\propto p_t(z_t^i|z^j)p_t(y^i|z^j)p(z^i)\label{eq:indiv_dist_unnorm}
\end{align}
Concretely, the posterior $ p_t(z^j|z_t^i, y^i)$ is directly proportional to the individual probabilities $p_t(z_t^i|z)p_t(y^i|z^j)p(z^j)$ up to a fixed normalizing factor in $z_t^i, y^i$.
Let $\hat{p}_t(z^j|z_t^i, y^i)=p_t(z_t^i|z^j)p_t(y^i|z^j)p(z^j)$ be the un-normalized approximation of $p_t(z^j|z_t^i,y^i)$. Now, define a (un-normalized) distribution $q_t(z^j|z_t^i, y^i)$, and sample $K\geq 1$ elements from the distribution to obtain $\{z^j\}_{j=1}^K\sim q_t(z^j|z_t^i, y^i)$. 
Leveraging self-normalizing importance sampling (SNIS), we define weights $\alpha_j$,
\begin{equation}
    \alpha_j = \frac{\hat{p}_t(z^j|z_t^i, y^i)}{q_t(z^j|z_t^i, y^i)} = \frac{p_t(z_t^i|z^j)p_t(y^i|z^j)p(z^j)}{q_t(z^j|z_t^i,y^i)}
\end{equation}
Defining $q_t(z^j|z_t^i,y^i)=p(z^j)$, $\alpha_j$ reduces to $p_t(z_t^i|z^j)p_t(y^i|z^j)$.
Of these, $p_t(z_t^i|z^j)$ is simply the conditional probability path, and $p_t(y^i|z^j)$ is the posterior-distribution of the condition given the target data point.
We then obtain the SNIS normalized weights by computing $w_j = \sfrac{\alpha_j}{\sum_{k=1}^K \alpha_k}$.
By definition, the SNIS estimator of $\mathbb{E}_{\{z^j\sim p_t(\cdot|z_t^i, y^i)\}}\left[g(z^j|z_t^i, y^i) \right]$ is $\mu_g^{(SNIS)}(z_t^i|K)=\sum_{j=1}^K w_j g(z^j | z_t^i, y^i)$.

We further define the Kish effective sample size (ESS) as, $ESS(z_t^i) = \sfrac{\left[\sum_{j=1}^K w_j \right]^2}{\sum_{j=1}^K w_j^2} = \sfrac{1}{\sum_{j=1}^K w_j^2}$
Thus, $1\leq ESS(z_t^i) \leq K$.
Putting all the pieces together, we obtain, 
\begin{align}
    \text{Var}\left(\mu_g^{(SNIS)}(z_t^i|K)\right) &= \sum_{j=1}^K w_j^2 \text{Var}\left(g(z^j | z^i_t, y^i)\right) = \sum_{j=1}^K w_j^2 \Sigma_g = \frac{\Sigma_g}{\text{ESS}(z_t^i)}
\end{align}
Now, observe that when $K=1$, the SNIS estimator \textit{reduces} to the gradient variance of the flow matching objective from Equation~\ref{eq:fm_theory}:
\begin{align}
    \mu_g^{(SNIS)}(z_t^i|1) &= \sum_{j=1}^1 w_j g(z^j|z_t^i, y^i) = g(z^j|z_t^i, y^i) \\
    &= \nabla_f ||f_\theta(z_t^i|t, y^i) - v(z_t^i|z^j)||^2 \\
    &\rightarrow \text{Var}\left(\mu_g^{(SNIS)}(z_t^i|1)\right) = \Sigma_g
\end{align}
Thus, by choosing $K\geq 1$, PAFM reduces the variance of each gradient estimate over $z_t^i$ by a factor of $\text{ESS}(z_t^i)\geq 1$ compared to flow matching. $\square$
\end{proof}

This further highlights an important result: PAFM is a \textit{generalization} of the flow matching objective, and reduces to it when $K=1$. 
\subsection{Model Training}\label{sec:pafm_training}
We now translate the formal definition of posterior-augmented flow matching (Equation~\ref{eq:pafm_theory}) to a practical training objective for rectified flow models. 
Posterior-augmented flow matching (PAFM) extends the standard flow matching (FM) objective by introducing three components: (i) a target set $\{z^j\}_{j=1}^K$ providing additional supervision trajectories, (ii) the conditional probability path $p_t(z_t^i|z^j)$ capturing how plausible reaching each target is from the current latent and (iii) the condition likelihood $p(y^i|z^j)$ measuring the compatibility of 
targets with the input condition.

\noindent{\textbf{Selecting target points $\{z^j\}_{j=1}^K$.}}
The choice of candidate target points is where PAFM offers its greatest flexibility, and the optimal strategy may vary across generative settings, data modalities and training regimes.
While the theoretical objective (Equation~\ref{eq:pafm_theory}) is defined over the true posterior $p_t(z^j|z_t^i,y^i)$, it is intractable and any practical approximation must be defined over a finite support set. 
PAFM is agnostic to how these candidates are sourced: they may be sampled from the target distribution, augmentations of data from the distribution, perturbations of data in the latent space, etc...
The importance weights $w_j$ then re-weight these candidates according to their posterior likelihoods, shaping the supervision signal toward the posterior distribution.
We investigate three proposal methods in Section~\ref{sec:implementation}, yet emphasize that PAFM is technically compatible with nearly any strategy.
\newline\textit{Note: we always control $\{z^j\}_{j=1}^K$ to include $z^i$ (i.e., the target data point from which $z_t^i$ is sampled under traditional FM).}

\noindent{\textbf{The conditional probability path, $p_t(z_t^i|z^j)$.}}
The conditional probability path describes how plausible each target $z^j$ is for a given latent-point $z_t^i$.
PAFM uses the exact conditional probability path from the FM assumptions inSection~\ref{sec:background}.
Recall that this is a Gaussian distribution parameterized by $\mathcal{N}(\beta_t z^j, \alpha_t^2 I_d)$ where $\alpha_t = t$ and $\beta_t = 1 - t$.
The log-likelhood is thus given by,
\begin{equation}
\log{p_t(z_t^i|z^j)} = -\frac{||z_t^i - (1-t)z^j||^2}{2t^2} + C\label{eq:cond_path}
\end{equation}
where $C$ is the normalizing constant.
In practice, we compute the un-normalized likelihood $\exp\left(-\sfrac{||z_t^i - (1-t)z^j||^2}{2t^2}\right)$.
\newline\textit{Note: Substituting $z_t^i=t\epsilon^i + (1-t)z^i$ into Equation~\ref{eq:cond_path} and simplifying yields the equivalent formulation: $\exp\left(-\sfrac{||t\epsilon^i + (1-t)(z^i - z^j)||^2}{2t^2}\right)$. 
Because $\{t, \epsilon^i, z^i\}$ are constant for all $z^j\in\{z^j\}_{j=1}^K$, the sharpness of $p_t(z_t^i|z^j)$ largely depends on the distance between each $z^j$ and $z^i$.}
\begin{figure}[!t] 
    \centering
    \includegraphics[width=.95\textwidth]{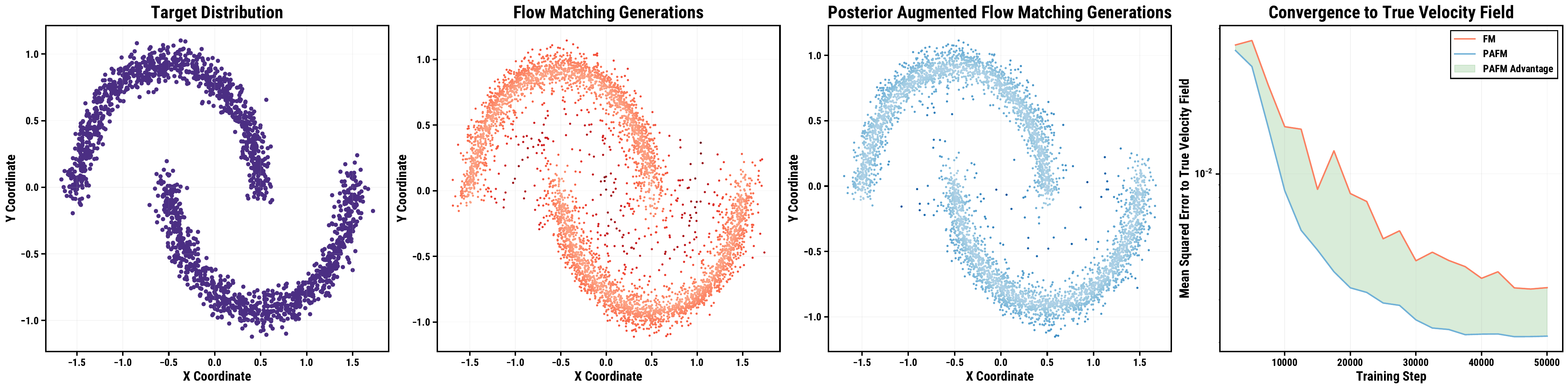}
    \caption{\textbf{Posterior augmented flow matching is more robust than flow matching.} We train two rectified flow models to generate two crescent moon distributions (left). Second-left: a model trained with FM generates many points in between the two moons. Second-right: the same model trained with PAFM has far less of an issue. Right: PAFM estimates the true velocity field across $t$ significantly better.}
    \label{fig:crescent_vf_error}
\end{figure}

\noindent{\textbf{The condition likelihood, $p_t(y^i|z^j)$.}}
The condition likelihood depicts how suitable each target $z^j$ is for the condition $y^i$.
In the case of class-conditioned flow matching, this is a deterministic function $\forall t$: $p_t(y^i|z^j) = \begin{cases}
1 \text{ if } y^i = y^j \\
0 \text{ otherwise }
\end{cases}$.
In the case of continuous-conditioned generation (e.g., text-to-image), this distribution is unknown apriori and must be approximated---forming the second design choice in PAFM. 
While many viable choices exist, one intuitive strategy for modeling the distribution is employing a vision-language model such as CLIP~\cite{clip} to compute the alignment of $y^i$ and $z^j$.

\noindent{\textit{Note: We use the naive class-conditioned variant for our text-to-image experiments owing to its simplicity, and leave more refined condition likelihood approximations to future work.}}

\noindent{\textbf{The PAFM training objective.}}
With a defined candidate selection strategy for $\{z^j\}_{j=1}^K$ and chosen $p_t(y^i|z^j)$ distribution, we can train with PAFM using the following general objective,
\begin{align}
    \min \overline{\mathcal{L}^{(PAFM)}}(\theta) &= \min \frac{1}{N}\sum_{i=1}^N \sum_{j=1}^K w_j || f_\theta(z_t^i| t, y^i) - v(z_t^i|z^j) ||^2_2 \\
    &\stackrel{\triangle}{=} \min\frac{1}{N}\sum_{i=1}^N || f_\theta(z_t^i| t, y^i) - \sum_{j=1}^K w_j v(z_t^i|z^j) ||^2_2\label{eq:pafm_training}
\end{align}
where $w_j = p_t(z^i_t|z^j)p_t(y^i|z^j), \text{ and } \sum_{j=1}^K w_j = 1$.

\noindent{\textbf{The PAFM batch step.}}
Algorithm~\ref{alg:pa_fm_pure} illustrates a batch step with PAFM, where \textcolor{NavyBlue}{navy text} denotes additions to the standard FM objective.

\begin{algorithm}
\caption{Posterior Augmented Flow Matching Batch Step}
\label{alg:pa_fm_pure}
\begin{algorithmic}[1]
\State \textbf{Input:} Model $f_\theta$, batch $B=\{(z^{i},y^{i},\epsilon^{i})\}_{i=1}^{n}$ where $(z^{i}, y^{i})\sim p_{\text{data}}(\cdot)$, $\epsilon^{i}\sim \mathcal{N}(0,\mathrm{I}_d)$, learning rate $\gamma$. 
\textcolor{NavyBlue}{Candidate pool $\{(z^{j}, y^{j})\}_{j=1}^{K}$ for each $i$ s.t.\ $(z^i, y^i) \in \{(z^j, y^j)\}_{j=1}^K$.
}
\State \textbf{Output:} Updated model parameters $\theta$
\State $\mathcal{L}(\theta) = 0$
\For{$i$ in range($n$)}
    \State $t\sim U[0,1),\quad z_t^{i} = t\epsilon^{i} + (1-t) z^{i}$
    \For{\textcolor{NavyBlue}{$j$ in range($K$)}}
        \State \textcolor{NavyBlue}{$v(z_t^{i}|z^{j}) = \sfrac{(z_t^{i} - z^{j})}{t},\quad \hat{w}_j = p_t(z^i_t|z^j)\, p(y^i|z^j)$}
    \EndFor
    \State $\mathcal{L}(\theta) \mathrel{+}= \| f_\theta(z_t^{i} | t, y^{i}) - \textcolor{NavyBlue}{\sum_{j=1}^{K} \left(\sfrac{\hat{w}_j}{\sum_{k=1}^K \hat{w}_k}\right) v(z_t^{i}|z^{j})}\|^2$
\EndFor
\State $\theta \leftarrow \theta - \frac{\gamma}{N}\nabla_\theta \mathcal{L}(\theta)$
\end{algorithmic}
\end{algorithm}

\noindent{\textbf{Understanding PAFM with an example.}}
To build intuition for PAFM's effect on training rectified flow models, we train two four-layer MLP models to generate a simple two-dimensional crescent moon distribution inFigure~\ref{fig:crescent_vf_error} (left). 
Under standard FM, the model generates a substantial number of spurious points in the region between the two crescents (second-left), despite training for an extensive period of time (50K steps).
This is a consequence of sparse supervision: the model observes few trajectories passing through these intermediate regions, thereby limiting its capability at estimating the true velocity field.
However, training the same model with PAFM largely resolves this artifact (second-right), as each intermediate point receives supervision from multiple weighted targets rather than a single trajectory---enabling the model better estimate the true velocity field.
The rightmost panel quantifies this gap directly: we plot the mean squared error between the learned velocity and the the analytically computed true velocity field across all denoising steps ($t$) over the course of training.
Notably, PAFM steadily lowers its velocity field error throughout training, whereas FM exhibits substantial variance in its field estimates across training steps. 
Moreover, PAFM ultimately converges to a substantially more accurate velocity field than FM. 
We provide additional implementation details in Appendix A. 
\section{Experiments}\label{sec:implementation}
Posterior-augmented flow matching (PAFM) is general and can be applied anywhere flow matching (FM) is used. 
As described in~Section~\ref{sec:pafm_training}, a primary design choice when training with PAFM is selecting the candidate target set $\{z^j\}_{j=1}^K$ for each $z^i$. 
While many proposal mechanisms are viable, we find that a simple nearest-neighbor retrieval in latent-space consistently improves over FM across class-conditioned (ImageNet-1K~\cite{imagenet}) and text-to-image (CC12M~\cite{sharma2018conceptual}) generation, across architectures (SiT~\cite{ma2024sit} and MMDiT~\cite{dit}), and across model scales (SiT-B/2 and SiT-XL/2).
We further show that PAFM can empirically reduce the gradient variance during training compared with FM, analyze the computational overhead of training with PAFM and ablate alternative target selection strategies---demonstrating the inherent tailorability of PAFM. 


\noindent{\textbf{Setup.}} 
We train all our models for 400K iterations with a batch size of 256 on images with 256x256 resolution, and with the REPA loss~\cite{yu2024representation}. 
All experiments are conducted on single nodes with 8 H100 GPUs.
We evaluate all our models without classifier-free guidance (CFG)~\cite{ho2022classifier} using the ODE sampler with 50 denoising steps, and report Inception Score (IS)~\cite{salimans2016improved}, Fréchet Inception Distance (FID)~\cite{heusel2017gans}, sFID~\cite{nash2021generating}, precision and recall over $50$K samples, as is standard.

\subsection{Nearest Neighbor PAFM Results}\label{sec:pafm_knn}
Our primary implementation of PAFM forms the target set $\{z^j\}_{j=1}^K$ by retrieving the $K$ nearest-neighbors of each training example $z^i$ in latent-space.
Neighbors are extracted \textit{once} for each $z^i$ using FAISS indices~\cite{douze2024faiss} computed over the train dataset during the data-preprocessing, and training simply loads the precomputed neighbors alongside each $z^i$---introducing no additional online compute.
Nearest neighbors allow a healthy effective sample size throughout training, as candidates are geometrically close to $z^i$.

\begin{wrapfigure}{r}{0.48\textwidth} 
  \centering
  \small
  \vspace{-14pt} 
  \captionof{table}{\textbf{PAFM improves FM on ImageNet-1K.} 
  Results on FID50K without CFG and NFE=50 after 400K training iterations.
  Models are trained with REPA.
  }
  \resizebox{\linewidth}{!}{%
  \begin{tabular}{lllccccc}
  \toprule
  \textbf{Model} & & \textbf{Obj.} & $\bf{K}$ & \textbf{FID}$\downarrow$ & \textbf{sFID}$\downarrow$ & \textbf{Prec.}$\uparrow$ & \textbf{Rec.}$\uparrow$ \\
  \midrule
  \multirow{4}{*}{SiT-B/2} & & \cellcolor{baseline}FM & \cellcolor{baseline}-- & \cellcolor{baseline}27.57 & \cellcolor{baseline}11.44 & \cellcolor{baseline}\underline{0.57} & \cellcolor{baseline}0.63 \\
          & & PAFM & 64  & 25.45 & 6.91 & \textbf{0.58} & \underline{0.64} \\
          & & PAFM & 16  & \textbf{24.88} & \textbf{6.81} & \textbf{0.58} & \underline{0.64} \\
          & & PAFM & 8   & \underline{25.25} & \underline{6.90} & \textbf{0.58} & \textbf{0.65} \\
          & & PAFM & 4   & 25.47 & 6.97 & \textbf{0.58} & \underline{0.64} \\
  \midrule
  \multirow{2}{*}{SiT-XL/2} & & \cellcolor{baseline}FM & \cellcolor{baseline}-- & \cellcolor{baseline}\underline{11.14} & \cellcolor{baseline}\underline{8.25} & \cellcolor{baseline}\underline{0.67} & \cellcolor{baseline}\textbf{0.66} \\
           & & PAFM & 16  & \textbf{9.85} & \textbf{5.24} & \textbf{0.68} & \underline{0.65} \\
  \bottomrule
  \end{tabular}
  }
  \label{tab:imagenet_results}
  \vspace{-10pt}
\end{wrapfigure}
\noindent{\textbf{Class-conditioned generation on ImageNet-1K.}} 
We restrict the neighbor search to images sharing the same class label, computing an independent FAISS index per class over the image latent encodings.
Since all candidates share the same class-label, the condition likelihood $p_t(y^i|z^j)=1, \forall z^j\in\{z^j\}_{j=1}^K$.
Table~\ref{tab:imagenet_results} shows the results when training SiT-B/2 and SiT-XL/2 models with PAFM compared to FM.
Notably, PAFM improves over FM-trained models in \textit{nearly all} metrics across \textit{both} model scales and \textit{all} $K$, highlighting its efficacy. 
Interestingly, PAFM peaks at $K=16$ for SiT-B/2 reflecting a bias-variance tradeoff: too few neighbors limit variance reduction while too many include distance candidates that add noise rather than signal.

\noindent{\textbf{Text-to-image generation on CC12M.}}
Without discrete class labels, we approximate class-conditioning in two stages.
First, we encode all images and captions using the SigLIP2~\cite{tschannen2025siglip} So400M/14 encoder\footnote{\url{https://huggingface.co/google/siglip2-so400m-patch14-384}} and retrieve the $M$ images whose encodings have highest cosine similarity to $y^i$.
Second, from this shortlist we select the $K$ nearest neighbors to $z^i$ in latent space, as in the class-conditioned setting.
This ensures candidates are both semantically compatible with the condition and geometrically close in latent space.
We use $M=128$ and $K=32$ in our experiments.
Since the candidate set has already been filtered for semantic compatibility with $y^i$, we treat the condition likelihood $p_t(y^i|z^j)$ as approximately uniform across the set and omit it from the weight computation.
Thus, only the conditional path likelihood $p_t(z_t^i|z^j)$ determines the relative importance weights---$w_j$.
We train MMDiT models on CC12M with PAFM and FM, and evaluate the resultant checkpoints after 400K training iterations on a held-out validation set of 50K examples.
Overall, we observe that FM achieves an FID50K of 10.37, compared to \textbf{9.45} when using PAFM, illustrating how PAFM can be utilized to improve flow matching even in dramatically more challenging text-to-image settings.

\subsection{Alternative Target Selection Strategies}\label{sec:alternative_pafm}
While selecting $\{z^j\}_{j=1}^K$ via nearest neighbor search 
is sufficient to consistently improve over flow matching, we emphasize that there may be many viable alternative strategies, depending on the desired application setting. 
In this section, we ablate two additional selection strategies, finding promising results with each. 
All models trained in this section are REPA-SiT-B/2 with 256 batch size on ImageNet-1K 256$\times$256 resolution for 400K iterations.

\noindent
\begin{wrapfigure}{r}{0.48\textwidth} 
\centering
\small
\vspace{-12pt}
\captionof{table}{
\textbf{PAFM with random augmentations improves FM.}
Results on FID50K without CFG and NFE=50 after 400K training iterations. 
The target set for PAFM is obtained by center cropping each training image to an ``augmented resolution'' (''Aug. Res.'') and then randomly resize cropping down to 256x256 resolution $K=5$ times.
}
\label{tab:source_transform}
\resizebox{\linewidth}{!}{%
\begin{tabular}{lccccc}
\toprule
\textbf{Obj.} & \textbf{Aug. Res.} & \textbf{FID}$\downarrow$ & \textbf{sFID}$\downarrow$ & \textbf{Prec.}$\uparrow$ & \textbf{Rec.}$\uparrow$ \\
\midrule
\cellcolor{baseline}FM & \cellcolor{baseline}$256\times256$ & \cellcolor{baseline}27.57 & \cellcolor{baseline}11.44 & \cellcolor{baseline}0.57 & \cellcolor{baseline}0.63 \\
\midrule
\multirow{4}{*}{PAFM} 
        & $281\times281$ & \underline{25.03} & \phantom{0}6.84 & \underline{0.58} & \textbf{0.65} \\
        & $320\times320$ & \textbf{24.15} & \textbf{\phantom{0}6.70} & \textbf{0.59} & \textbf{0.65} \\
        & $384\times384$ & 25.43 & \phantom{0}7.13 & \underline{0.58} & \textbf{0.65} \\
        & $512\times512$ & 25.44 & \phantom{0}\underline{6.72} & \underline{0.58} & \underline{0.64} \\
\bottomrule
\end{tabular}
}
\end{wrapfigure}%
\noindent{\textbf{Augmenting the source image.}}
A natural alternative to creating the target set for each training image $z^i$ by retrieving its nearest neighbors, is by augmenting the training image itself.
Depending on the augmentation, resultant images may provide different views (e.g., from random crops), styles (e.g., jitter or applying style transfer methods), or compositions (e.g., horizontal flips, rotations) of the original image---increasing the diversity of the underlying supervision signal. 
In this first foray, we focus on the effect of sampling different spatial views of $z^i$.
Specifically, we first center-crop each image to $256R\times 256R$ resolution, where $R\geq 1$ is referred to as the ``augmentation scale'' (aug. scale).
We then take $K-1$ random resize crops down to $256\times 256$ resolution, resulting in a target set of $K-1$ different ``views'', and finish by adding $z^i$ to the set. 
Notably, the augmentation scale controls the spatial diversity of the resulting targets: small factors yield images that are nearly identical to $z^i$, while larger factors produce more varied views.
While this strategy rests on the assumption that random crops are plausible samples from the target image distribution, we posit that it is reasonable given its widespread use for training foundation vision models~\cite{clip, oquab2024dinov, tschannen2025siglip}.
Table~\ref{tab:source_transform} shows results on SiT-B/2 with $K=5$.
PAFM improves over FM at every augmentation scale, with 1.25$\times$ substantially lowering FID by 3.42.
We compute these augmentations online in our experiments; however they could equally be precomputed and instead directly loaded during training to eliminate this overhead.

\begin{wrapfigure}{r}{0.48\textwidth}
\centering
\small
\vspace{-12pt}
\captionof{table}{
\textbf{Sampling K times from the VAE improves FM.}
Results on FID50K without CFG and NFE=50 after 400K training iterations. 
The target set for PAFM is obtained by randomly sampling from the training image's VAE moment $K$ times (FM always samples once).
}
\resizebox{\linewidth}{!}{%
\label{tab:vae_sampling}
\definecolor{baseline}{gray}{0.90}
\begin{tabular}{clccccc}
\toprule
\textbf{Model} & \textbf{Obj.} & \textbf{K} & \textbf{FID}$\downarrow$ & \textbf{sFID}$\downarrow$ & \textbf{Prec.}$\uparrow$ & \textbf{Rec.}$\uparrow$ \\
\midrule
\multirow{2}{*}{SiT-B/2} & \cellcolor{baseline}FM & \cellcolor{baseline}-\phantom{0} & \cellcolor{baseline}27.57 & \cellcolor{baseline}11.44 & \cellcolor{baseline}0.57 & \cellcolor{baseline}0.63 \\
\cmidrule{2-7}
& PAFM & 10 & \textbf{25.15} & \textbf{6.95} & \textbf{0.58} & \textbf{0.65} \\
\bottomrule
\end{tabular}
}
\vspace{-10pt}
\end{wrapfigure}
\noindent{\textbf{Sampling from the VAE moment distribution.}}
All our rectified flow matching models are trained in the latent space of a VAE~\cite{ldm}, where targets are created by first encoding each training image into a multivariate Gaussian distribution (termed ``moment'') and then sampling from it---amounting to small random perturbations around the mean encoding.
Flow matching (and all our previously described PAFM approaches) sample once from the moment at each step.
Thus, a very simple target selection strategy for PAFM is instead sampling from this moment $K$ times.
While this is the cheapest selection strategy proposed (sampling from the moment requires minimal overhead), the VAE posterior is tightly concentrated. 
This yields candidates that explore only a small neighborhood in latent space, and can limit their diversity.
Table~\ref{tab:vae_sampling} shows that this simple strategy can improve over flow matching when $K=10$. 
This shows how PAFM can extract gains from even very modest transformations of the underlying target point $z^i$, underscoring its flexibility as a training framework.

\subsection{Analysis}\label{sec:analysis}
\noindent
\noindent{\textbf{PAFM induces negligible computational overhead.}}
We benchmark the computational overhead of training with the nearest neighbor version of PAFM compared to FM using the MMDiT architecture on CC12M over 500 consecutive training iterations. 
\begin{wrapfigure}{r}{0.48\textwidth}
\centering
\small

\captionof{table}{
\textbf{PAFM marginally increases overhead.}
Computational overhead comparison between FM and PAFM with $K{=}32$ neighbors on CC12M 256$\times$256 with MMDiT. Benchmarked on 8$\times$H100 GPUs with total batch size of 256.}
\label{tab:computational_overhead}

\setlength{\tabcolsep}{4pt}
\resizebox{\linewidth}{!}{%
\begin{tabular}{@{}lcc@{}}
\toprule
\textbf{Metric} & \textbf{FM} & \textbf{PAFM} \\
\midrule  
Throughput (samples/sec) & 1151 & 1074 {\footnotesize\color{NavyBlue!60}(--6.6\%)} \\
Peak Memory (GB) & 19.11 & 19.18 {\footnotesize\color{NavyBlue!60}(+0.4\%)} \\
FLOPs/step (GFLOPs) & 10773 & 10773 {\footnotesize\color{NavyBlue!60}(+0.0\%)} \\
\bottomrule
\end{tabular}
}

\vspace{12pt}

\includegraphics[width=\linewidth]{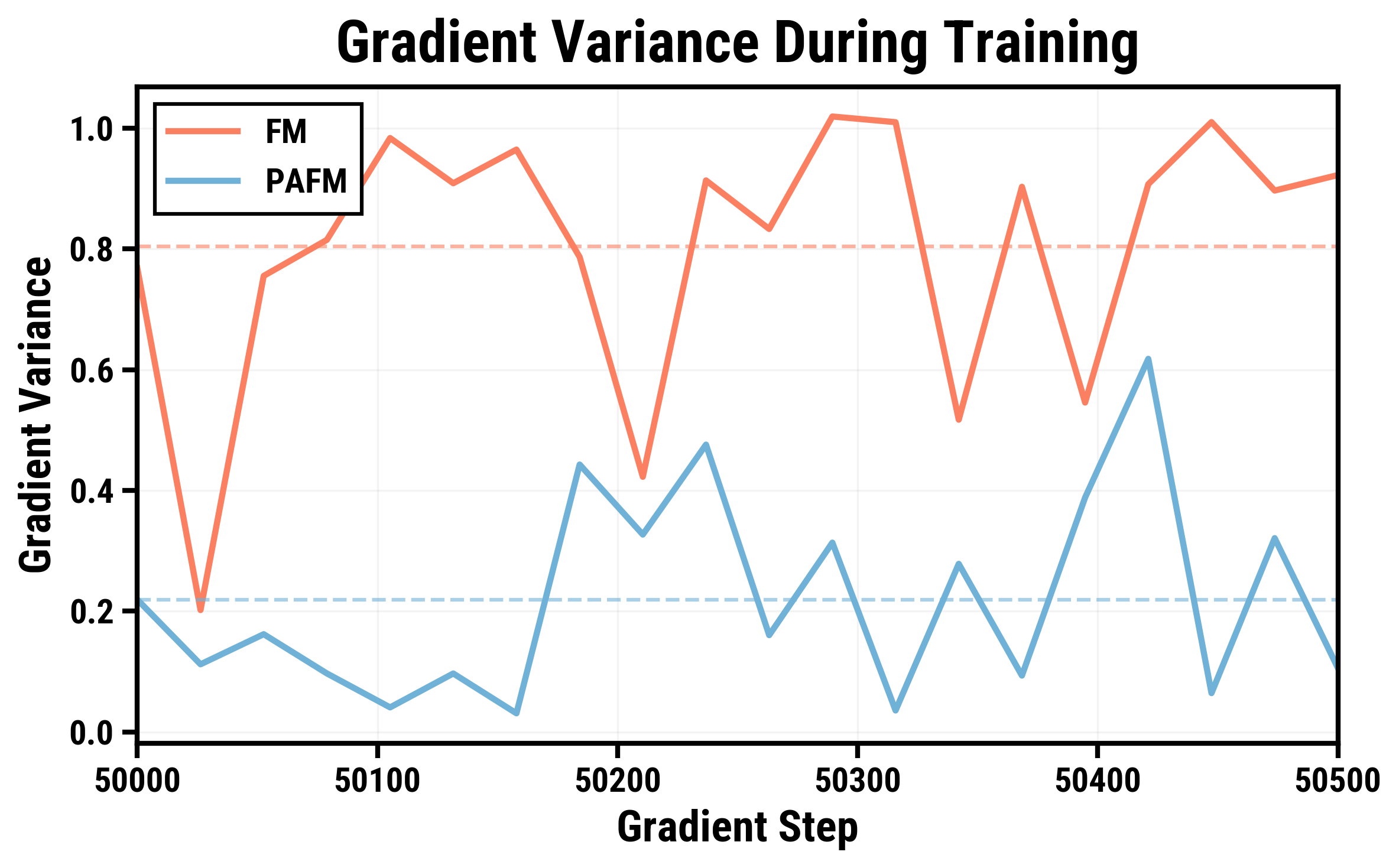}

\captionof{figure}{
\textbf{PAFM reduces mini-batch gradient variance over FM.}
Optimization steps over the same 500 iterations for REPA-SiT-B/2~\cite{ma2024sit} models trained with nearest neighbor PAFM with $K{=}16$ and FM on ImageNet-1K~\cite{imagenet}.
Full lines show gradient variance at each iteration, while the correspondingly colored dashed lines indicate the mean variance across all iterations.
PAFM reduces gradient variance by $\sim4\times$.
}
\label{fig:gradient_variance}

\vspace{-20pt}
\end{wrapfigure}
PAFM marginally decreases throughput by 6.6\%, only increases peak memory consumption by 0.4\% and has no perceived change in GFLOPs.
We attribute this to the fact that PAFM nearest neighbors can be computed \textit{once} during data-preprocessing and simply loaded alongside each training example during training. 
Thus, PAFM can be nearly as efficient as its flow matching counterpart.

\noindent{\textbf{PAFM reduces gradient variance in practice.}} 
Theorem~\ref{theorem:variance} establishes that under---certain sampling conditions---posterior-augmented flow matching (PAFM) lowers the gradient variance at each intermediate point in the latent space, by a factor dependent on the number of additional samples used for supervision compared to flow matching (FM). 
By lower bounding the FM gradient variance at arbitrary fixed intermediate points, PAFM should further lower bound the variance of the gradients across batches during training, leading to more stable gradient updates. 
We test this empirically by training two REPA-SiT-B/2 models on ImageNet-1K 256x256 resolution with the FM and nearest neighbor PAFM (K=16) objectives respectively.
We measure gradient variance as $\mathrm{Tr}(\Sigma_g) = \mathbb{E}\!\left[\|g - \hat g\|^2\right]$, where $\Sigma_g$ is the covariance of the stochastic gradient defined in Theorem~\ref{theorem:variance} and $\hat g$ is the full-batch gradient.
At each training iteration, we randomly sample 500 batches (each of size 256), and compute the variance across these mini-batch gradients.
Figure~\ref{fig:gradient_variance} compares the gradient variance midway through training at 50,000 steps over the same 500 iterations. 
PAFM achieves lower batch-gradient variance compared to FM \textit{at every iteration}, and its mean variance is nearly $4\times$ less than that of FM: $0.22\text{, versus }0.80$.
This illustrates how PAFM can reduce gradient variance beyond theoretical settings; it is observed empirically in real-world settings, yielding more stable gradient steps.

\section{Conclusion}\label{sec:conclusion}

We introduce Posterior-Augmented Flow Matching (PAFM), a principled extension of flow matching (FM) that replaces the sparse one-to-one supervision of FM with a posterior-weighted expectation over all plausible trajectories. 
We show that PAFM is an unbiased estimator of the FM objective that provably lower bounds its gradient variance under certain sampling conditions. 
PAFM can be flexibly adapted for efficiently training rectified flow models in real-world settings, improving FM performance by up to 3.4 FID50K across class-conditioned and text-to-image benchmarks (ImageNet-1K and CC12M respectively), popular model architectures (SiT and MMDiT), and model scales (SiT-B/2 and SiT-XL/2).---whilst increasing computational overhead by just 6.6\%. 
\section{Acknowledgments}\label{sec:acknowledgments}
This work is supported in part by an NSF Graduate Research Fellowship, a Google PhD Fellowship, and NSF awards \#2622839 and \#2403297.

\bibliographystyle{plainnat}
\bibliography{main}

@String(ICCV= {Int. Conf. Comput. Vis.})

@String(ECCV= {Eur. Conf. Comput. Vis.})

@String(ICLR = {Int. Conf. Learn. Represent.})

@String(ICCV  = {ICCV})

@String(ECCV  = {ECCV})

@String(ICLR  = {ICLR})

@misc{yu2024representation,
      title={Representation Alignment for Generation: Training Diffusion Transformers Is Easier Than You Think}, 
      author={Sihyun Yu and Sangkyung Kwak and Huiwon Jang and Jongheon Jeong and Jonathan Huang and Jinwoo Shin and Saining Xie},
      year={2024},
      eprint={2410.06940},
      archivePrefix={arXiv},
      primaryClass={cs.CV},
      url={https://arxiv.org/abs/2410.06940}, 
}

@inproceedings{ma2024sit,
    author = {Ma, Nanye and Goldstein, Mark and Albergo, Michael S. and Boffi, Nicholas M. and Vanden-Eijnden, Eric and Xie, Saining},
    title = {SiT: Exploring Flow and Diffusion-Based Generative Models with Scalable Interpolant Transformers},
    year = {2024},
    url = {https://doi.org/10.1007/978-3-031-72980-5_2},
    booktitle = {Computer Vision – ECCV 2024: 18th European Conference, Milan, Italy, September 29–October 4, 2024, Proceedings, Part LXXVII},
}

@misc{ma2024sitexploring,
      title={SiT: Exploring Flow and Diffusion-based Generative Models with Scalable Interpolant Transformers}, 
      author={Nanye Ma and Mark Goldstein and Michael S. Albergo and Nicholas M. Boffi and Eric Vanden-Eijnden and Saining Xie},
      year={2024},
      eprint={2401.08740},
      archivePrefix={arXiv},
      primaryClass={cs.CV},
      url={https://arxiv.org/abs/2401.08740}, 
}

@misc{ho2022classifier,
      title={Classifier-Free Diffusion Guidance}, 
      author={Jonathan Ho and Tim Salimans},
      year={2022},
      eprint={2207.12598},
      archivePrefix={arXiv},
      primaryClass={cs.LG},
      url={https://arxiv.org/abs/2207.12598}, 
}

@inproceedings{
lipman2022flow,
title={Flow Matching for Generative Modeling},
author={Yaron Lipman and Ricky T. Q. Chen and Heli Ben-Hamu and Maximilian Nickel and Matthew Le},
booktitle={ICLR},
year={2023},
url={https://openreview.net/forum?id=PqvMRDCJT9t}
}

@article{albergo2023stochastic,
  title={Stochastic interpolants: A unifying framework for flows and diffusions},
  author={Albergo, Michael S and Boffi, Nicholas M and Vanden-Eijnden, Eric},
  journal={arXiv},
  year={2023}
}

@article{fid,
  title={Gans trained by a two time-scale update rule converge to a local nash equilibrium},
  author={Heusel, Martin and Ramsauer, Hubert and Unterthiner, Thomas and Nessler, Bernhard and Hochreiter, Sepp},
  journal={Advances in neural information processing systems},
  volume={30},
  year={2017}
}

@article{ddim,
  title={Denoising diffusion implicit models},
  author={Song, Jiaming and Meng, Chenlin and Ermon, Stefano},
  journal={arXiv},
  year={2020}
}

@article{sfid,
  title={Generating images with sparse representations},
  author={Nash, Charlie and Menick, Jacob and Dieleman, Sander and Battaglia, Peter W},
  journal={arXiv},
  year={2021}
}

@inproceedings{ldm,
  title={High-resolution image synthesis with latent diffusion models},
  author={Rombach, Robin and Blattmann, Andreas and Lorenz, Dominik and Esser, Patrick and Ommer, Bj{\"o}rn},
  booktitle={Proceedings of the IEEE/CVF conference on computer vision and pattern recognition},
  pages={10684--10695},
  year={2022}
}

@inproceedings{dit,
  title={Scalable diffusion models with transformers},
  author={Peebles, William and Xie, Saining},
  booktitle={Proceedings of the IEEE/CVF International Conference on Computer Vision},
  pages={4195--4205},
  year={2023}
}

@misc{sd3,
  title={Scaling Rectified Flow Transformers for High-Resolution Image Synthesis}, 
  author={Patrick Esser and Sumith Kulal and Andreas Blattmann and Rahim Entezari and Jonas Müller and Harry Saini and Yam Levi and Dominik Lorenz and Axel Sauer and Frederic Boesel and Dustin Podell and Tim Dockhorn and Zion English and Kyle Lacey and Alex Goodwin and Yannik Marek and Robin Rombach},
  year={2024},
  eprint={2403.03206},
  archivePrefix={arXiv},
  primaryClass={cs.CV}
}

@article{var,
  title={Visual Autoregressive Modeling: Scalable Image Generation via Next-Scale Prediction},
  author={Tian, Keyu and Jiang, Yi and Yuan, Zehuan and Peng, Bingyue and Wang, Liwei},
  journal={arXiv},
  year={2024}
}

@inproceedings{imagenet,
  title={Imagenet: A large-scale hierarchical image database},
  author={Deng, Jia and Dong, Wei and Socher, Richard and Li, Li-Jia and Li, Kai and Fei-Fei, Li},
  booktitle={2009 IEEE conference on computer vision and pattern recognition},
  pages={248--255},
  year={2009},
  organization={Ieee}
}

@inproceedings{clip,
  title={Learning transferable visual models from natural language supervision},
  author={Radford, Alec and Kim, Jong Wook and Hallacy, Chris and Ramesh, Aditya and Goh, Gabriel and Agarwal, Sandhini and Sastry, Girish and Askell, Amanda and Mishkin, Pamela and Clark, Jack and others},
  booktitle={International conference on machine learning},
  pages={8748--8763},
  year={2021},
  organization={PMLR}
}

@article{
oquab2024dinov,
title={{DINO}v2: Learning Robust Visual Features without Supervision},
author={Maxime Oquab and Timoth{\'e}e Darcet and Th{\'e}o Moutakanni and Huy V. Vo and Marc Szafraniec and Vasil Khalidov and Pierre Fernandez and Daniel HAZIZA and Francisco Massa and Alaaeldin El-Nouby and Mido Assran and Nicolas Ballas and Wojciech Galuba and Russell Howes and Po-Yao Huang and Shang-Wen Li and Ishan Misra and Michael Rabbat and Vasu Sharma and Gabriel Synnaeve and Hu Xu and Herve Jegou and Julien Mairal and Patrick Labatut and Armand Joulin and Piotr Bojanowski},
journal={TMLR},
year={2024},
}

@article{esser2024scaling,
author = {Esser, Patrick and Kulal, Sumith and Blattmann, Andreas and Entezari, Rahim and M\"{u}ller, Jonas and Saini, Harry and Levi, Yam and Lorenz, Dominik and Sauer, Axel and Boesel, Frederic and Podell, Dustin and Dockhorn, Tim and English, Zion and Rombach, Robin},
title = {Scaling rectified flow transformers for high-resolution image synthesis},
year = {2024},
journal = {ICML},
}

@inproceedings{sharma2018conceptual,
  title = {Conceptual Captions: A Cleaned, Hypernymed, Image Alt-text Dataset For Automatic Image Captioning},
  author = {Sharma, Piyush and Ding, Nan and Goodman, Sebastian and Soricut, Radu},
  booktitle = {Proceedings of ACL},
  year = {2018},
}

@article{stoica2025contrastive,
  title={Contrastive Flow Matching},
  author={Stoica, George and Ramanujan, Vivek and Fan, Xiang and Farhadi, Ali and Krishna, Ranjay and Hoffman, Judy},
  journal={ICCV},
  year={2025}
}

@article{fukumizu2024flow,
  title={Flow matching achieves almost minimax optimal convergence},
  author={Fukumizu, Kenji and Suzuki, Taiji and Isobe, Noboru and Oko, Kazusato and Koyama, Masanori},
  journal={arXiv},
  year={2024}
}

@inproceedings{
huang2025improving,
title={Improving Flow Matching by Aligning Flow Divergence},
author={Yuhao Huang and Taos Transue and Shih-Hsin Wang and William M Feldman and Hong Zhang and Bao Wang},
booktitle={Forty-second International Conference on Machine Learning},
year={2025},
url={https://openreview.net/forum?id=FeZimuj6SG}
}

@inproceedings{aggarwal2001surprising,
  title={On the surprising behavior of distance metrics in high dimensional space},
  author={Aggarwal, Charu C and Hinneburg, Alexander and Keim, Daniel A},
  booktitle={International conference on database theory},
  year={2001},
  organization={Springer}
}

@article{reu2025gradient,
  title={Gradient Variance Reveals Failure Modes in Flow-Based Generative Models},
  author={Reu, Teodora and Dromigny, Sixtine and Bronstein, Michael and Vargas, Francisco},
  journal={arXiv},
  year={2025}
}

@article{agrawalla2025floq,
  title={floq: Training critics via flow-matching for scaling compute in value-based rl},
  author={Agrawalla, Bhavya and Nauman, Michal and Agrawal, Khush and Kumar, Aviral},
  journal={arXiv},
  year={2025}
}

@article{wiegand1968kish,
  title={Kish, L.: Survey Sampling. John Wiley \& Sons, Inc., New York, London 1965, IX + 643 S., 31 Abb., 56 Tab., Preis 83 s.},
  author={H. Wiegand},
  journal={Biometrische Zeitschrift},
  year={1968}
}

@article{heusel2017gans,
  title={Gans trained by a two time-scale update rule converge to a local nash equilibrium},
  author={Heusel, Martin and Ramsauer, Hubert and Unterthiner, Thomas and Nessler, Bernhard and Hochreiter, Sepp},
  journal={Advances in neural information processing systems},
  volume={30},
  year={2017}
}

@article{tong2024cfm,
  title   = {Improving and Generalizing Flow-Based Generative Models with Minibatch Optimal Transport},
  author  = {Tong, Alexander and Fatras, Kilian and Malkin, Nikolay and Huguet, Guillaume and Zhang, Yanlei and others},
  journal = {Transactions on Machine Learning Research},
  year    = {2024}
}

@inproceedings{liu2022rectified,
  title     = {Flow Straight and Fast: Learning to Generate and Transfer Data with Rectified Flow},
  author    = {Liu, Xingchao and Gong, Chengyue and Liu, Qiang},
  booktitle = {Advances in Neural Information Processing Systems},
  year      = {2022}
}

@inproceedings{lee2023curvature,
  title     = {Minimizing Trajectory Curvature of {ODE}-based Generative Models},
  author    = {Lee, Sangyun and Kim, Beomsu and Ye, Jong Chul},
  booktitle = {International Conference on Machine Learning},
  year      = {2023}
}

@inproceedings{esser2024rectified,
  title     = {Scaling Rectified Flow Transformers for High-Resolution Image Synthesis},
  author    = {Esser, Patrick and Kulal, Sumith and Blattmann, Andreas and Entezari, Rahim and M{\"u}ller, Jonas and others},
  booktitle = {International Conference on Machine Learning},
  year      = {2024}
}

@inproceedings{huang2025divergence,
  title     = {Improving Flow Matching by Aligning Flow Divergence},
  author    = {Huang, Yuhao and Transue, Taos and Wang, Shih-Hsin and Feldman, William M. and Zhang, Hong and Wang, Bao},
  booktitle = {International Conference on Machine Learning},
  year      = {2025}
}

@article{hertrich2025rectified,
  title   = {On the Relation between Rectified Flows and Optimal Transport},
  author  = {Hertrich, Johannes and Chambolle, Antonin and Delon, Julie},
  journal = {arXiv},
  year    = {2025}
}

@inproceedings{ho2020ddpm,
  title     = {Denoising Diffusion Probabilistic Models},
  author    = {Ho, Jonathan and Jain, Ajay and Abbeel, Pieter},
  booktitle = {Advances in Neural Information Processing Systems},
  year      = {2020}
}

@inproceedings{song2021score,
  title     = {Score-Based Generative Modeling through Stochastic Differential Equations},
  author    = {Song, Yang and Sohl-Dickstein, Jascha and Kingma, Diederik P. and Kumar, Abhishek and Ermon, Stefano and Poole, Ben},
  booktitle = {International Conference on Learning Representations},
  year      = {2021}
}

@article{salimans2016improved,
  title={Improved techniques for training gans},
  author={Salimans, Tim and Goodfellow, Ian and Zaremba, Wojciech and Cheung, Vicki and Radford, Alec and Chen, Xi},
  journal={Advances in neural information processing systems},
  year={2016}
}

@article{nash2021generating,
  title={Generating images with sparse representations},
  author={Nash, Charlie and Menick, Jacob and Dieleman, Sander and Battaglia, Peter W},
  journal={arXiv},
  year={2021}
}

@article{douze2024faiss,
  title={The Faiss library},
  author={Matthijs Douze and Alexandr Guzhva and Chengqi Deng and Jeff Johnson and Gergely Szilvasy and Pierre-Emmanuel Mazaré and Maria Lomeli and Lucas Hosseini and Hervé Jégou},
  year={2024},
  journal={arXiv},
}

@article{tschannen2025siglip,
  title={SigLIP 2: Multilingual Vision-Language Encoders with Improved Semantic Understanding, Localization, and Dense Features},
  author={Tschannen, Michael and Gritsenko, Alexey and Wang, Xiao and Naeem, Muhammad Ferjad and Alabdulmohsin, Ibrahim and Parthasarathy, Nikhil and Evans, Talfan and Beyer, Lucas and Xia, Ye and Mustafa, Basil and H\'enaff, Olivier and Harmsen, Jeremiah and Steiner, Andreas and Zhai, Xiaohua},
  year={2025},
  journal={arXiv preprint arXiv:2502.14786}
}
\clearpage
\section*{Appendix}
\addcontentsline{toc}{section}{Appendix} 
\setcounter{page}{1}
\renewcommand\thesection{\Alph{section}}  
\renewcommand\thesubsection{\thesection.\arabic{subsection}}
\renewcommand\thefigure{\thesection.\arabic{figure}}  
\renewcommand\thetable{\thesection.\arabic{table}}  
\setcounter{section}{0}
\setcounter{figure}{0}
\setcounter{table}{0}
\setcounter{algorithm}{0}
\makeatletter
\makeatother


\section{Crescent Example Implementation Details}\label{app:crescent}
\noindent\textbf{Setting.} We construct a two-dimensional crescent moon target distribution, each comprising of 1,000 points.
The source distribution is an isotropic Gaussian with a standard deviation of 0.1, intentionally shifted away from the target to make the transport problem non-trivial.

\noindent\textbf{Model Architecture.} We train two models, each with the flow matching (FM) and posterior augmented flow matching (PAFM) objectives respectively.
Each model is a 4-layer MLP, with the SiLU activation function
after each intermediate layer, and a hidden dimension of 128. 
We use 32-dimensional sinusoidal time-step embeddings, to illustrate that PAFM extends beyond linear time-steps.
The time embedding and data coordinates are concatenated to form the input to the first layer of the network.

\noindent\textbf{Training.}
Both models are trained for 50,000 steps with the Adam optimizer
at an initial learning rate of 5e-4, with cosine learning rate decay. 
For PAFM, we compute the posterior weights $w_j$ over all $N=$2,000 target data points at each training step.
Both models are initialized with the same random seed.

\noindent\textbf{Sampling.}
We generate samples using Euler integrations for 300 steps, and generate 5,000 samples for each model.

\noindent\textbf{Evaluation.}
We quantify convergence to the true velocity field by analytically computing the marginal velocity field and measuring the mean squared error against each model's predictions, averaged across denoising time steps $t$.

\section{PAFM improves FM without REPA}\label{app:vanilla}
\begin{wrapfigure}{r}{0.48\textwidth} 
  \centering
  \small
  \vspace{-15pt} 
  \captionof{table}{\textbf{PAFM improves FM on ImageNet-1K without REPA.} 
  Results on FID50K without CFG and NFE=50 after 400K training iterations.
  }
  \resizebox{\linewidth}{!}{%
  \begin{tabular}{lllccccc}
  \toprule
  \textbf{Model} & & \textbf{Obj.} & $\bf{K}$ & \textbf{FID}$\downarrow$ & \textbf{sFID}$\downarrow$ & \textbf{Prec.}$\uparrow$ & \textbf{Rec.}$\uparrow$ \\
  \midrule
  \multirow{4}{*}{SiT-B/2} & & \cellcolor{baseline}FM & \cellcolor{baseline}-- & \cellcolor{baseline}39.98 & \cellcolor{baseline}7.07 & \cellcolor{baseline}\underline{0.51} & \cellcolor{baseline}\textbf{0.64} \\
          & & PAFM & 64  & \textbf{38.19} & \textbf{6.65} & \textbf{0.52} & \underline{0.63} \\
          & & PAFM & 16  & 39.03 & 6.97 & \underline{0.51} & \underline{0.63} \\
          & & PAFM & 4   & 38.80 & 6.91 & \textbf{0.52} & \underline{0.63} \\
  \bottomrule
  \end{tabular}
  }
  \label{tab:imagenet_vanilla}
\end{wrapfigure}
We train SiT-B/2 models without REPA on ImageNet-1K with 256 batch size for 400K iterations. 
We evaluate all models without classifier-free guidance (CFG), using the ODE sampler for 50 denoising steps, and report FID, sFID, precision and recall, over 50,000 generated samples.
Table~\ref{tab:imagenet_vanilla} shows the results.
Overall, we observe that PAFM improves over FM, despite not utilizing REPA. 

\section{PAFM May be More Robust to Data Sparsity}\label{app:sparsity}

We train two rectified flow models for 15,000 steps, using the same regime described inAppendix~\ref{app:crescent}, across different data-sparsity regimes.
In each regime, we vary the number of target distribution samples available during training according to $N=\{100, 200, 500, 1000\}$ total samples.
Figure~\ref{fig:sparsity_crescents} illustrates the results, where each plot displays the kernel density estimate (KDE) over generations in each sparsity-setting.
Overall, we find that models trained with PAFM better recover the true target distribution under data sparsity.
We attribute this to the dense supervision inherent in PAFM: at each latent point, the objective aggregates supervision from many target samples, rather than from a single instance as in FM.
This provides substantially richer gradient signal, even with $N$ is small.

\begin{figure}[!t] 
    \centering
    \includegraphics[width=.95\textwidth]{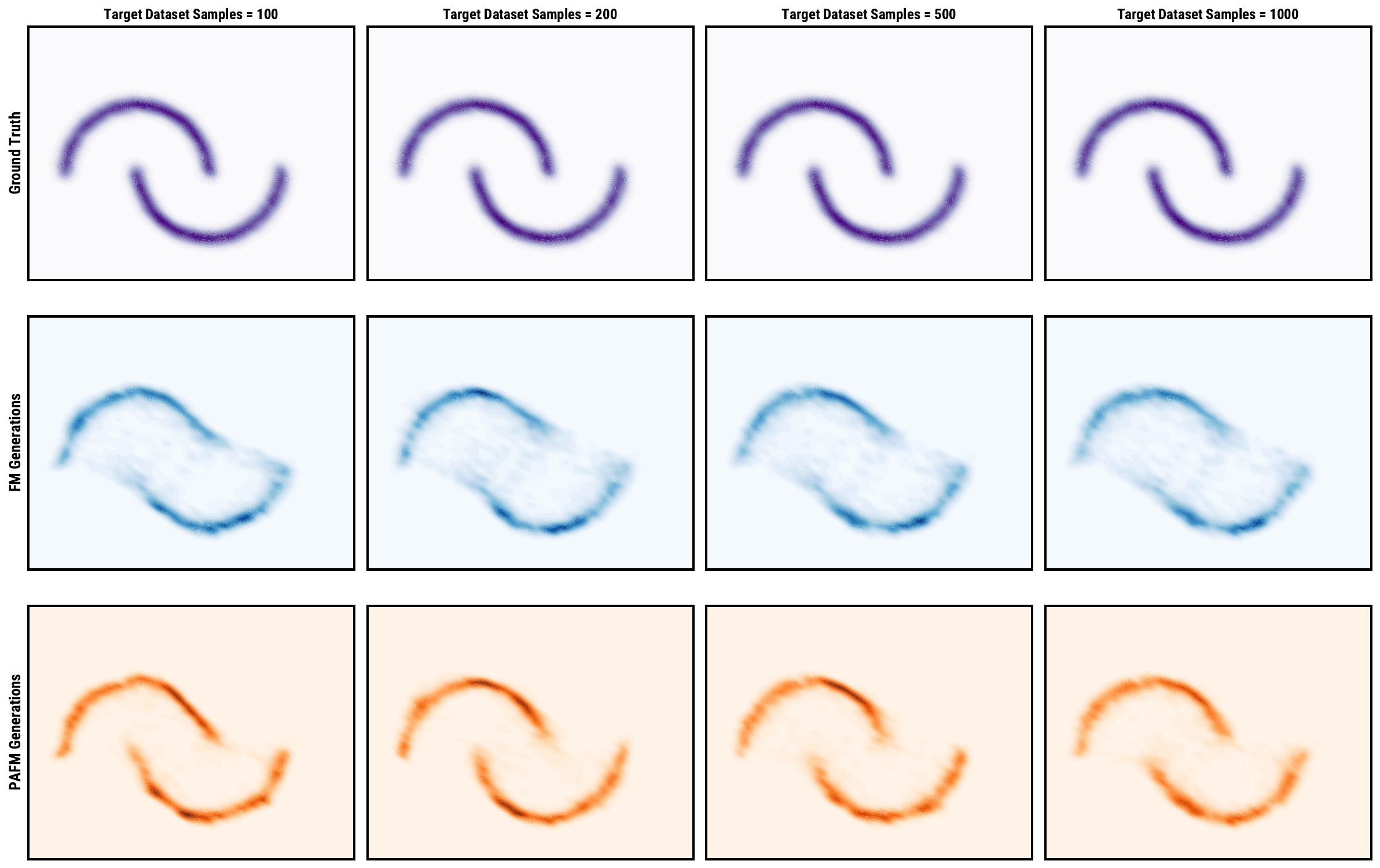}
    \caption{\textbf{Posterior augmented flow matching learns the target distribution better under data sparsity.} We train two rectified flow models to generate two crescent moon distributions, under different data sparsity regimes. Column titles indicate the total number of training data points given to each model. Rows indicate the source used for generating targets: the ``Ground Truth'' distribution (top), the trained ``FM'' model (middle), and the trained ``PAFM'' model (bottom). Each plot shows kernel density estimates (KDE) for the generated target distribution.}
    \label{fig:sparsity_crescents}
\end{figure}

\end{document}